\documentclass[letterpaper, 10 pt, conference]{ieeeconf}  % Comment this line out
\newtheorem{theorem}{Theorem}[section]

\newtheorem{prop}[theorem]{Proposition}

\newtheorem{problem}{Problem}
\newtheorem{definition}[theorem]{Definition}
\newtheorem{rem}[theorem]{Remark}
\newtheorem{ex}[theorem]{Example}
\IEEEoverridecommandlockouts
\usepackage{cite}
\usepackage{algorithm}
\usepackage[noend]{algpseudocode}
\usepackage{multirow}
\usepackage{color}
\usepackage{amsfonts}
\usepackage{mysymbol}
\usepackage[dvipsnames]{xcolor}
\usepackage[hyphens]{url}
\usepackage{subcaption}
\usepackage{amsmath}
\usepackage{amssymb}
\allowdisplaybreaks
\newcommand{\vertiii}[1]{{\left\vert\kern-0.25ex\left\vert\kern-0.25ex\left\vert #1
    \right\vert\kern-0.25ex\right\vert\kern-0.25ex\right\vert}}
\usepackage[breaklinks=true, colorlinks, bookmarks=true, citecolor=Black, urlcolor=Violet,linkcolor=Black]{hyperref}
\usepackage{comment}
\usepackage[colorinlistoftodos,prependcaption,textwidth=1.5cm,textsize=tiny]{todonotes}

\usepackage[font={small}]{caption}

\title{\LARGE \bf
Verified Compositions of Neural Network Controllers \\ for Temporal Logic Control Objectives
}

\author{Jun Wang, Samarth Kalluraya and Yiannis Kantaros% <-this % stops a space
\thanks{The authors are with the Department of Electrical and Systems Engineering, Washington University in St. Louis, St. Louis, MO. Email: {\tt\small \{junw,k.samarth,ioannisk\}@wustl.edu}}%
}

\begin{document}
\maketitle
\thispagestyle{empty}
\pagestyle{empty}

\begin{abstract}
This paper presents a new approach to design verified compositions of Neural Network (NN) controllers for autonomous systems with tasks captured by Linear Temporal Logic (LTL) formulas. Particularly, the LTL formula requires the system to reach and avoid certain regions in a temporal/logical order. We assume that the system is equipped with a finite set of trained NN controllers. Each controller has been trained so that it can drive the system towards a specific region of interest while avoiding others. Our goal is to check if there exists a temporal composition of the trained NN controllers - and if so, to compute it -  that will yield composite system behaviors that satisfy a user-specified LTL task for any initial system state belonging to a given set. To address this problem, we propose a new approach that relies on a novel integration of automata theory and recently proposed reachability analysis tools for NN-controlled systems. We note that the proposed method can be applied to other controllers, not necessarily modeled by NNs, by appropriate selection of the reachability analysis tool. We focus on NN controllers due to their lack of robustness. The proposed method is demonstrated on navigation tasks for aerial vehicles.
\end{abstract}
%\vspace{-0.2cm}
\section{INTRODUCTION}
\vspace{-0.1cm}
%Recently, neural networks (NNs) have been successfully applied to many challenging tasks ranging from image processing\cite{He_2016_CVPR} to designing controllers for autonomous systems \cite{bansal2020combining} 
% and aircraft collision avoidance\cite{julian2019verifying} 
%with significant performance benefits. In autonomous systems, NNs are typically used as feedback controllers, motion planners, or perception modules.
%— There has been an increasing interest in using neural networks in closed-loop control systems to improve performance and reduce computational costs for on-line implementation.
%Recent advances in Neural Networks (NNs) coupled with technological advances in sensing, computation, and control have resulted in major achievements in autonomous systems \cite{grigorescu2020survey,lin2021deep} wherein NNs are often used as feedback controllers, motion planners, or perception modules.
Several methods have been proposed recently to train Neural Network (NN) controllers for autonomous systems. Such training methods include e.g., deep reinforcement learning (RL) \cite{gao2019reduced} and model predictive control (MPC) \cite{rubies2019classification}.
Despite the high real-time performance of NN-driven systems, they typically lack safety and robustness guarantees as underscored by recent studies \cite{huang2017adversarial}. %To address this limitation several safe reinforcement learning (RL) methods have been proposed that %\cite{goodfellow2015explaining,sun2021vulnerability, tesla2016accident, uber2018accident}.
%
% To address this limitation, several safe Reinforcement Learning (RL) methods that aim to ensure safety during either training \cite{carr2022safe} or deployment \cite{Bastani2021SafeRL} based on shielding.
% %
% %The main idea in these works is to compute a backup policy (called shield) that overrides the nominal control policy when the system is about to violate safety constraints. 
% Although safety is ensured in these works, liveness/performance guarantees are not considered. 
%
% To address this limitation, formal methods have been leveraged recently to design RL controllers %\cite{alshiekh2018safe, gao2019variance, hasanbeig2019reinforcement, bouton2019reinforcement, bozkurt2020control, cai2021reinforcement,kantaros2022accelerated}. 
% \cite{cai2021reinforcement,kantaros2022accelerated}. 
% %
% % Particularly, the main goal in these works is to learn a safe controller that maximizes the probability of satisfying complex properties captured by formal languages such as Linear Temporal Logic (LTL) \cite{kantaros2018text, wang2021verifying}.
% %
% Nevertheless, the designed controllers are not supported by any safety or robustness guarantees when function approximations, such as NNs, are used to model them.
%
To address this limitation, various methods have been proposed to verify robustness properties of trained NN controllers  \cite{fazlyab2020safety, Dutta2017OutputRA}. 
For instance, \cite{Dutta2017OutputRA} addresses the problem of output range analysis of trained NNs given an input set. 
Safety verification of dynamical systems with feedback NN controllers has been studied as well in \cite{huang2019reachnn, sun2019formal, Hu2020ReachSDPRA, Tran2020NNVTN, dutta2019reachability, ivanov2021verisig,sun2022formal} and the references therein.
Typically, these methods investigate if a dynamical system with a trained NN controller can satisfy a  reach-avoid property given a set of possible initial states. %For instance, %\cite{ivanov2021verisig} considers hybrid systems driven by feedforward NN controllers with sigmoid or tanh activation functions. Similarly,
%\cite{Hu2020ReachSDPRA} considers discrete-time linear time-varying systems with NN feedback controllers with arbitrary non-linear activation functions. 
% 
%Common in \cite{huang2019reachnn, sun2019formal, Hu2020ReachSDPRA, Tran2020NNVTN, dutta2019reachability, ivanov2021verisig, ivanov2020verifying, tran2019safety,sun2022formal} 
%is that they focus on reach-avoid properties, i.e., they aim to verify that the system can eventually reach a desired state while avoiding unsafe states given a set of initial states. 

Common in all the above works is that  they focus on learning or verifying the efficiency/safety of a \textit{single} NN controller for a given task. However, the sample complexity and computational cost of learning a \textit{single} NN controller increases drastically as the task complexity increases. Motivated by these limitations, compositional RL methods have been proposed recently that aim to learn a set of base NN controllers which are then composed to satisfy 
a complex task captured by temporal logics; see e.g., \cite{jothimurugan2021compositional,tasse2022skill,neary2022verifiable}. The key idea in these works is to decompose the task into simpler sub-tasks for which NN controllers can be learned, using RL, more efficiently. Then, these NN controllers are composed to satisfy the original task. Nevertheless, the resulting controllers often either lack safety guarantees or the provided guarantees are impractical (e.g., lower bounds on satisfaction probability) for safety-critical applications; also, these guarantees are typically specific to a fixed and given initial system state. To address this issue, in this paper we propose a new method to design \textit{verified} temporal compositions of trained NN controllers for temporal logic tasks.
Related is also the recent work in \cite{ivanov2021composelearning} which, however, unlike the above papers and ours, does not consider temporal logic tasks. For instance, in \cite{ivanov2021composelearning}, the sub-tasks are revealed by the environment while in \cite{jothimurugan2021compositional,tasse2022skill,neary2022verifiable}
the sub-tasks are `strategically' selected to satisfy a temporal logic task. %captured by formal languages. 

%\vspace{-1cm}
Specifically, in this paper, we consider autonomous systems tasked with complex high-level missions captured by a fragment of Linear Temporal Logic (LTL), called co-safe LTL \cite{baier2008principles}.  We assume that the system is governed by discrete-time linear time-varying dynamics and that the LTL task requires reaching and avoiding certain regions in a temporal/logical order. Also, the system has access to a finite set of already trained controllers modeled as NNs. Each controller is trained so that it can drive the system towards a specific region of interest while avoiding others. We do not make any assumptions about how these NNs have been trained; for instance they may have been trained using RL or MPC-based methods. Our goal is to check if there exists a temporal composition of these NN controllers - and if so, to compute it -  that will yield composite system behaviors that always satisfy a user-specified LTL task for any initial system state belonging to a given set.
%if there exists verified temporal compositions of these trained NN controllers so that the resulting closed-loop system satisfies the LTL specification for any initial state belonging to a given set. 
To address this problem, we leverage automaton representations of LTL formulas as well as graph-search methods and existing reachability analysis for NN-driven systems \cite{Hu2020ReachSDPRA}. 
%
% Particularly, inspired by recent compositional temporal logic planning methods \cite{bisoffi2018hybrid, kantaros2020reactive}, we use the automaton representation of the LTL formulas and graph-search methods allows us to decompose the LTL task into sequences of reach-avoid tasks. Then, we employ recently proposed reachability methods\cite{Hu2020ReachSDPRA} to check if there exists at least one temporal composition of the base controllers that can verifiably satisfy the sub-tasks captured by at least one sequence of tasks. 
%
%To the best of our knowledge, this is the first work that designs verified compositions of NN controllers that yield system behaviors satisfying complex and long-term tasks captured by formal languages.
%
We note that our approach can handle any other open-loop or feedback controllers that are not necessarily modeled as NNs by appropriate selection of the reachability analysis method.
%verifies LTL properties for dynamical systems with NN controllers.
%
In this paper we focus on NN controllers due to their fragility to  imperceptible input perturbations \cite{huang2017adversarial}. 
%
%Yet, our proposed method can be coupled with any controllers by appropriate selection of reachability analysis methods.

%\vspace{-0.1cm}
\textbf{Contributions:} \textit{First}, we propose a new approach to design verified temporal compositions of NN controllers for co-safe LTL tasks. \textit{Second}, we show correctness of the proposed method and discuss trade-offs between completeness and computational efficiency. \textit{Third}, we demonstrate the efficiency of the proposed approach on several navigation tasks that involve aerial vehicles.

\vspace{-0.1cm}
\section{Problem Formulation}\label{sec:problem}
\vspace{-0.1cm}
\textbf{Closed-loop system:} 
We consider discrete-time linear systems defined as follows:
%Consider the following discrete-time systems:
\begin{equation}\label{eq:dynamics}
    \bbx_{t+1} =\bbf(\bbx_t, \bbu_t)=\bbA_t\bbx_t + \bbB_t\bbu_t+\bbc_t,
\end{equation}
where $\bbx_t\in\ccalX\subseteq\mathbb{R}^d$ and $\bbu_t\in\ccalU_t\subseteq\mathbb{R}^n$ denote the state and the control input of the system at time $t\geq0$, respectively. Also, $\bbA_t\in\mathbb{R}^{d\times d}, \bbB_t\in\mathbb{R}^{d\times n}$ are the system matrices while $\bbc_t\in\mathbb{R}^d$ is an exogenous input. We assume that state-space $\ccalX$ contains $N>0$ sub-spaces denoted by $\ell_i \subset \ccalX$ modelling regions of interest or unsafe areas. %\footnote{The assumption of disjoint sub-spaces can be relaxed straightforwardly; see Section \ref{sec:prune}. %Also, as it will be shown in Section , having disjoint sub-spaces can reduce the computational effort required for verification.} 
%Such sub-spaces can model regions of interest or unsafe areas.  
Also, we assume that at any time $t$ the system can apply  control inputs selected from a finite set of feedback controllers collected in the set $\Xi=\{\xi_{i}\}_{i=1}^N$, where $\xi_{i}(\bbx_t): \ccalX\to\mathbf{R}^n$ maps system states to control actions. \textit{We assume that the controller $\xi_{i}$ is selected by the system when, given any initial state in $\ccalX$, the system state  $\bbx_t$ needs to be driven towards the interior of $\ell_i$.}
We consider cases where the controllers $\xi_{i}(\bbx_t)$ are  parameterized by multi-layer feed-forward fully-connected neural networks (NNs). Such NN controllers can be implemented using available methods; see e.g., \cite{chen2022large}. %\footnote{\textcolor{red}{[if we have space, we can describe in math the structure of such NN. otherwise, a reference would suffice]}} 
Hereafter, with slight abuse of notation, we denote by $\xi(t)$ the controller selected from $\Xi$ at time $t$.
To ensure that NN output respects the input constraint, we consider a projection operator, denoted by $\text{Proj}_{\ccalU_t}$, and define the control input as $\bbu_t=\text{Proj}_{\ccalU_t}\xi(t)$. We denote the closed-loop system with dynamics \eqref{eq:dynamics} and the projected NN control policy as:
%
%NN-based closed loop system at time $t$ can be defined as: %$\bbx_{t+1}=\bbf(\bbx_t, \text{Proj}_{\ccalU_t}\xi(t))$.
\begin{equation}\label{eq:closedloop}
    \bbx_{t+1}=\bbf_{\xi}(\bbx_t)
    %\bbf(\bbx_t, \text{Proj}_{\ccalU_t}\xi(t)):
\end{equation}
Next, we define a high-level NN-based control strategy $\boldsymbol\xi$ as a \textit{temporal composition} of the controllers in $\Xi$.
\begin{definition}[Control Strategy]\label{eq:NNcontrol}
A NN-based control strategy $\boldsymbol\xi$ is defined as a finite sequence of NN controllers selected from $\Xi$, i.e., $\boldsymbol\xi=\mu(0),\mu(1),\mu(2),\dots,\mu(K)$, for some finite $K>0$, where $\mu(k)\in\Xi$, for all $k\in\{0,1,\dots,K\}$, and $\mu(k)$ is applied for a finite horizon $H_k$. 
\end{definition}
We note again that $\mu(k)$ is a feedback controller from $\Xi$. For instance, if $\mu(0)=\xi_i$, for some $i\in\{1,\dots,N\}$, then the system applies the controller $\xi_i(\bbx_t)$, $\forall t\in\{0,1,\dots,H_0\}$. Given $\boldsymbol\xi$ and an initial state $\bbx_0$, the corresponding closed-loop system \eqref{eq:closedloop} generates a finite sequence of system states, denoted by $\tau(\bbx_0)=\bbx_0,\bbx_1,\dots,\bbx_t,\dots, \bbx_F$, where $F=\sum_{k=0}^KH_k$.

%\subsection{Linear Temporal Logic Properties}
\textbf{Linear Temporal Logic Properties:}  We define mission and safety properties for the system \eqref{eq:dynamics} using Linear Temporal Logic (LTL) as it allows to specify a wide range of high-level tasks  \cite{leahy2016persistent,kantaros2018text}. LTL consists of atomic propositions (i.e., Boolean variables), denoted by $\mathcal{AP}$, Boolean operators, (i.e., conjunction $\wedge$, and negation $\neg$), and two temporal operators, next $\bigcirc$ and until $\mathcal{U}$. LTL formulas over a set $\mathcal{AP}$ can be constructed based on the following grammar: $\phi::=\text{true}~|~\pi~|~\phi_1\wedge\phi_2~|~\neg\phi~|~\bigcirc\phi~|~\phi_1~\mathcal{U}~\phi_2$, where $\pi\in\mathcal{AP}$. For brevity we abstain from presenting the derivations of other Boolean and temporal operators, e.g., \textit{always} $\square$, \textit{eventually} $\lozenge$, \textit{implication} $\Rightarrow$, which can be found in \cite{baier2008principles}. 
Hereafter, we define the set $\mathcal{AP}$ as $\mathcal{AP}=\cup_i\{\pi^{\ell_i}\}$, where $\pi^{\ell_i}$ is an atomic predicate that is true when the system state $\bbx_t$ is within region $\ell_i$.
We restrict our attention to co-safe LTL properties that exclude the use of the `always' operator. Co-safe LTL formulas are satisfied by discrete finite plans $\tau$ defined as finite sequences of system states $\bbx_t\in\ccalX$. i.e., $\tau(\bbx_0)=\bbx_0,\bbx_1,\dots,\bbx_t,\dots, \bbx_F$, where $F>0$ denotes a finite horizon \cite{baier2008principles}. Given $\boldsymbol\xi$ and an initial state $\bbx_0$, we say that the closed-loop system \eqref{eq:closedloop} satisfies $\phi$, denoted by $\bbf_{\xi}\models\phi$, if \eqref{eq:closedloop} generates a sequence $\tau(\bbx_0)$ that satisfies $\phi$. %Satisfaction of an LTL formula $\phi$ by $\tau$ is denoted by $\tau\models\phi$, where $\models$ is a satisfaction relation \cite{baier2008principles}. 

%with the following structure
%\begin{equation}\label{eq:phi}
%    \phi=\phi_{\text{mission}}\wedge\phi_{\text{safe}},
%\end{equation}
%where (i) $\phi_{\text{mission}}$ is a co-safe LTL formula, that excludes the use of the `always' operator; and (ii) $\phi_{\text{safe}}$ is an LTL formula requiring all-time avoidance of the unsafe state-space $O$, i.e., $\phi_{\text{safe}}=\square\neg \pi^O$.

\begin{ex}
Examples of co-safe LTL specifications follow: (i) $\phi=\Diamond(\pi^{\ell_1})\wedge (\neg \pi^{\ell_2}\ccalU \pi^{\ell_1})$ captures a common reach-avoid property requiring the system to eventually reach the region of interest $\ell_1$ while avoiding in the meantime the unsafe region $\ell_2$; (ii) $\phi=(\Diamond\pi^{\ell_1})\wedge (\Diamond\pi^{\ell_2}) \wedge (\Diamond\pi^{\ell_3})\wedge(\Diamond\pi^{\ell_4})\wedge (\neg \pi^{\ell_4} \ccalU \pi^{\ell_1})$ requires the system to eventually reach the regions $\ell_1$, $\ell_2$, $\ell_3$, and $\ell_4$, in any order, as long as $\ell_4$ is avoided until region $\ell_1$ is reached.
\end{ex}

\begin{problem}\label{pr1}%[Verification Problem]
\textit{Given} (i) a set of initial states $\ccalX_0\subseteq\ccalX$; (ii) the system dynamics \eqref{eq:dynamics}; (iii) a co-safe LTL property $\phi$; and (iv) a finite set $\Xi$ of NN controllers, check if there exists a NN-based control strategy $\boldsymbol\xi$ (Definition \ref{eq:NNcontrol}), so that $\bbf_\xi\models\phi$, for all $\bbx_0\in\ccalX_0$; if there exists such a $\boldsymbol\xi$, compute it as well.
\end{problem}

\begin{rem}[Controllers for LTL tasks]
Several reinforcement learning (RL) methods have been proposed that can train a \textit{single} controller, that can be parameterized by a NN, to satisfy an LTL task $\phi$; see e.g., \cite{gao2019reduced,hahn,kantaros2022accelerated,bozkurt2020control} and the references therein. Verification of a single NN controller (as opposed to the set of controllers considered here) 
with respect to an LTL formula can be accomplished by existing reachability analysis tools for discrete-time systems; see e.g., \cite{Hu2020ReachSDPRA}. Particularly, a sequence of reachable sets, capturing all possible system states at future time instants $t$, needs to be computed under the considered NN controller and a set of initial states $\ccalX_0$. Then, it suffices to check if the atomic predicates that are satisfied across this sequence construct a word that can be accepted by an automaton of $\phi$ \cite{baier2008principles}. %In particular, it suffices to convert $\phi$ into an automaton and the check if this automaton accepts $w$ \cite{baier2008principles}.
\end{rem}

\vspace{-0.2cm}
\section{Verified Compositions of NN Controllers for co-safe LTL Tasks}
% \textcolor{red}{[Add a summary of the method first (short paragraph)]}
% \textcolor{red}{[Add an algorithmic step-by-step description of the proposed tool; we should refer to specific lines of the algorithm in the text ]}

\vspace{-0.1cm}
%We propose a new method to verify co-safe LTL properties for NN-based systems.
% see Alg. \ref{algo1}.
%
Our approach to solve Problem \ref{pr1} consists of the following steps.
First, we translate the LTL formula $\phi$ into a Deterministic Finite state Automaton (DFA); see Section \ref{sec:DFA}. %Then, we prune the DFA to reduce the computational effort required for verification; see Section \ref{sec:prune}. 
Second, by leveraging the DFA, we decompose $\phi$ into reach-avoid sub-tasks; see Sections \ref{sec:reachAvoid}-\ref{sec:VerReach}. Then, we apply graph-search methods combined with reachability analysis to check if there exists $\boldsymbol\xi$ so that $\bbf_\xi\models\phi$, for all $\bbx_0\in\ccalX_0$; see Section \ref{sec:vltl}. Trade-offs between completeness and computational efficiency are discussed in Section \ref{sec:complete}.

% \begin{algorithm}[h]
% \footnotesize
% \caption{Verified Composition of NN Controllers}\label{algo1}
% \begin{algorithmic}[1]
% \State \textbf{Input}: Formula $\phi$; System Dynamics \eqref{eq:dynamics}; Controllers $\Xi$\;
% \State Translate $\phi$ into DFA $D$\; (Section \ref{sec:DFA})
% %\State Prune DFA $D$\; (see Section \ref{sec:prune})
% \State Using $D$, decompose $\phi$ into reach-avoid properties\; (Sections \ref{sec:reachAvoid}-\ref{sec:VerReach})
% \State Apply graph-search \& reachability analysis over the DFA state-space \label{a_dfs} (Section \ref{sec:vltl}-Alg. \ref{algo2})
% \State \textbf{Output}: Verification result  $R\in\{[\texttt{True},\boldsymbol\xi],\texttt{False}\}$\;
% \end{algorithmic}
% \end{algorithm}%\vspace{-0.5cm}
% \normalsize

\vspace{-0.2cm}
\subsection{From LTL formulas to DFA}\label{sec:DFA}
\vspace{-0.2cm}

First, we translate  $\phi$, constructed using $\mathcal{AP}$, into a DFA defined as follows \cite{baier2008principles}; see also Fig. \ref{fig:example}.%
\begin{definition}[DFA]
A Deterministic Finite state Automaton (DFA) $D$ over $\Sigma=2^{\mathcal{AP}}$ is defined as a tuple $D=\left(\ccalQ_{D}, q_{D}^0,\Sigma,\delta_D,q_D^F\right)$, where $\ccalQ_{D}$ is the set of states, $q_{D}^0\in\ccalQ_{D}$ is the initial state, $\Sigma$ is an alphabet, $\delta_D:\ccalQ_D\times\Sigma\rightarrow\ccalQ_D$ is a deterministic transition relation, and $q_D^F\in\ccalQ_{D}$ is the accepting/final state. 
\end{definition}

\begin{figure}[h]
\centering
\begin{subfigure}[t]{0.48\linewidth}
\centering
  \includegraphics[width=0.8\linewidth]{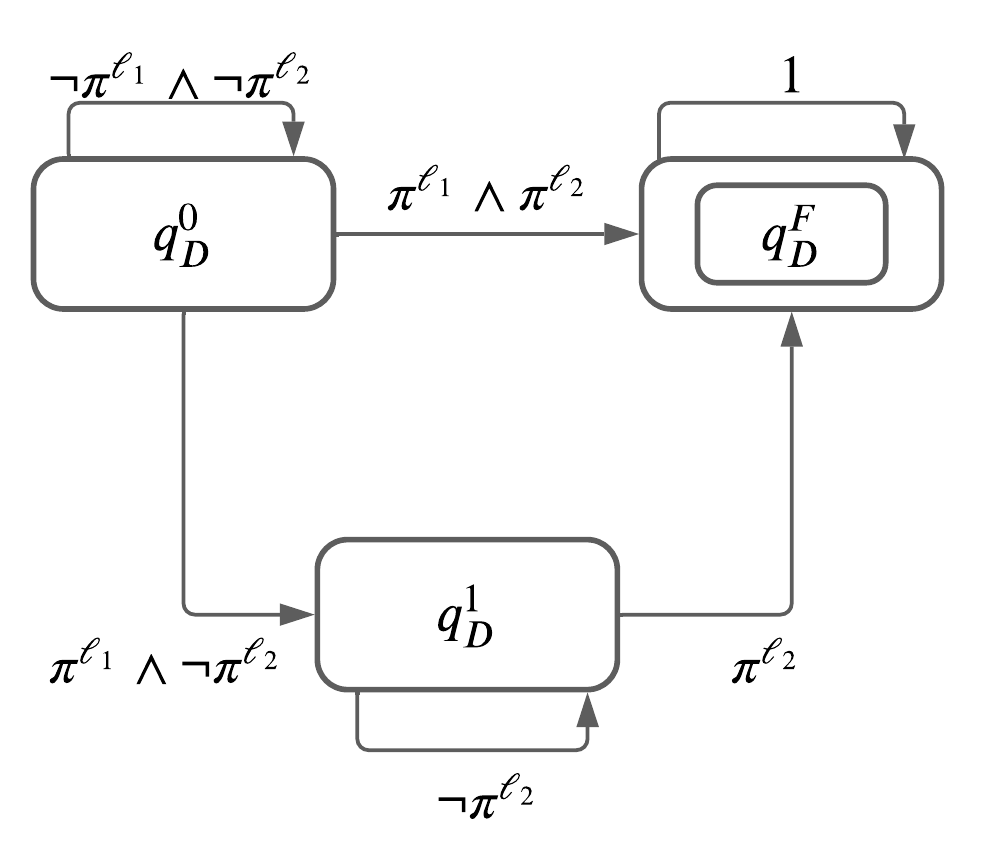}
    \caption{DFA}
    \label{fig:example}
\end{subfigure}
\hfill
\begin{subfigure}[t]{0.48\linewidth}
\centering
   \includegraphics[width=0.8\linewidth]{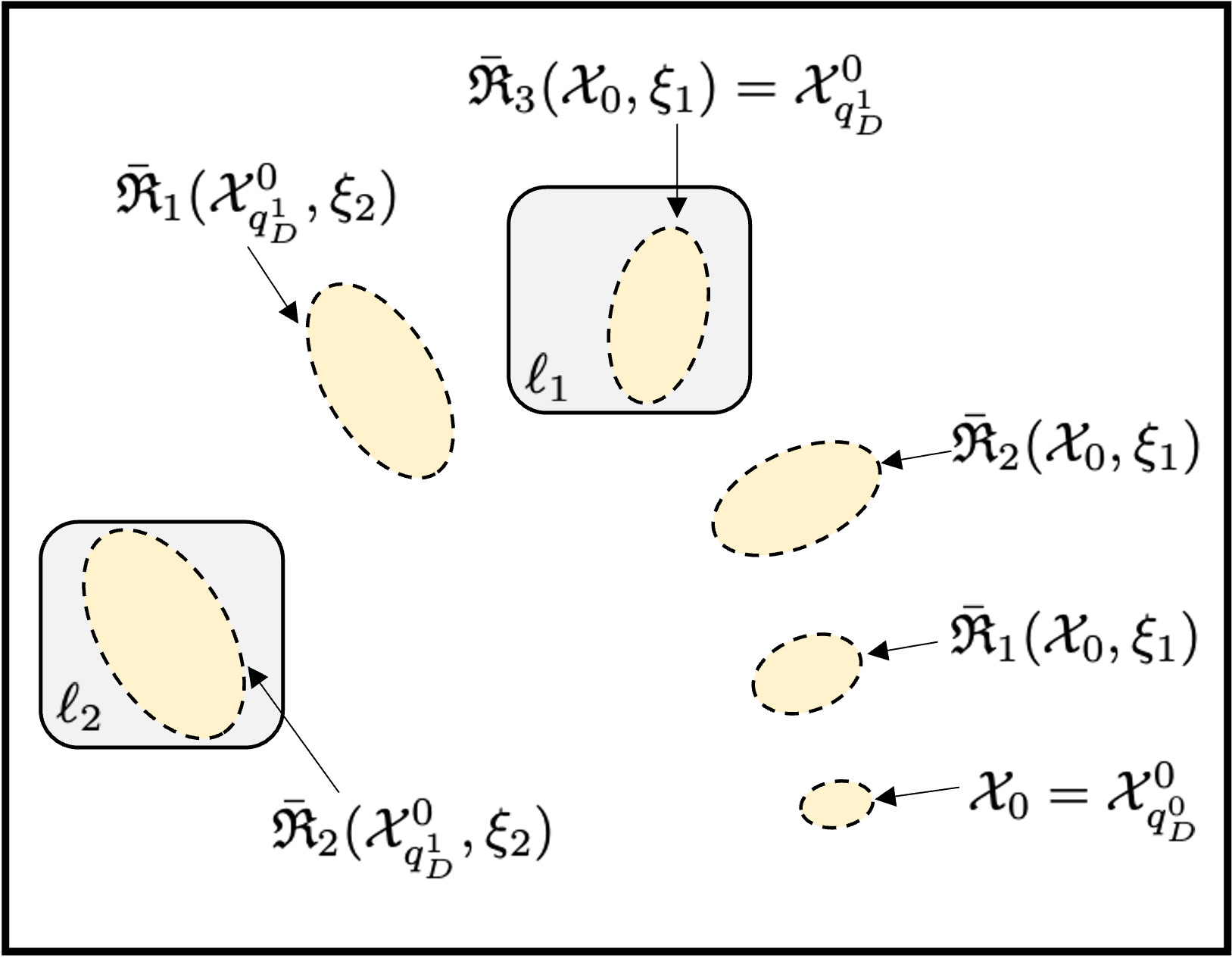}
    \caption{Reachable sets}
    \label{fig:reach}
\end{subfigure}
\caption{Fig \ref{fig:example} shows the DFA corresponding to $\phi=\Diamond(\pi^{\ell_2}) \wedge (\neg \pi^{\ell_2} \ccalU \pi^{\ell_1})$. Fig. \ref{fig:reach} illustrates the reachability analysis over the DFA state space (see Section \ref{sec:vltl}).}
\label{fig:case1}\vspace{-0.5cm}
\end{figure}

To interpret a temporal logic formula over a sequence $\tau(\bbx_0)$ generated by \eqref{eq:closedloop}, we use a labeling function $L:\ccalX\rightarrow 2^{\mathcal{AP}}$ that maps system states to symbols $\sigma\in2^{\mathcal{AP}}$. A finite sequence of states $\tau(\bbx_0)=\bbx_0,\bbx_1,\dots, \bbx_F$ satisfies $\phi$ if the \textit{word} $w=L(\bbx_0)L(\bbx_1)\dots L(\bbx_F)$ yields an accepting DFA run, i.e., if starting from the initial state $q_D^0$, each element in $w$ yields a DFA transition so that the final state $q_D^F$ is reached \cite{baier2008principles}.  Note that a DFA can be constructed using existing tools such as \cite{fuggitti-ltlf2dfa}.

\vspace{-0.2cm}
\subsection{From DFA to Reach-Avoid Properties}\label{sec:reachAvoid}
\vspace{-0.2cm}

Given any DFA state $q_D\in\ccalQ_D$, we compute a set $\ccalR_{q_D}$ that collects all DFA states that can be reached, in one hop, from $q_D$ using a symbol $\sigma$. In math, we have:
\vspace{-0.1cm}
\begin{equation}\label{eq:R}
    \ccalR_{q_D}=\{q_D'~|~q_D'=\delta_D(q_D,\sigma), \sigma\in\Sigma\}.
\end{equation}
Then, given $q_D$ and for each  DFA state $q_D'\in\ccalR_{q_D}$, we introduce the following definitions. We construct a set that collects the states $\bbx\in\ccalX$, so that if the system state coincides with one of these states, then a symbol $\sigma=L(\bbx)\in\Sigma$ enabling this DFA transition will be generated. We collect these states in the set $\ccalX_{q_D\rightarrow q_D'}$, i.e.,
\vspace{-0.1cm}
\begin{equation}\label{eq:X}
    \ccalX_{q_D\rightarrow q_D'} =\{\bbx\in\ccalX ~|~q_D'=\delta_D(q_D,L(\bbx))\}
\end{equation}
%\vspace{-0.1cm}
In what follows, for simplicity, we assume that for all $q_D\in\ccalQ_D$ there exists a feasible self-loop around every DFA state $q_D$; later, in Section \ref{sec:VerReach}, we relax this assumption. Starting from any system state $\bbx\in\ccalX$ and a DFA state $q_D$, transition from $q_D$ to $q_D'\in\ccalR_{q_D}\setminus q_D$ will eventually occur after $H_{q_D}\geq0$ discrete time steps, if (i) the system state $\bbx_t$ remains within $\ccalX_{q_D\rightarrow q_D}$ for the next $H_{q_D}-1$ steps, and (ii) at time $t=H_{q_D}$, we have that $\bbx_{H_{q_D}}\in\ccalX_{q_D\rightarrow q_D'}$. Essentially, (i)-(ii) model a \textit{reach-avoid} requirement. 
For instance, (i) may require a robot to stay within the obstacle-free space and (ii) may require a robot to eventually enter a region.

\begin{ex}[Reach-Avoid Properties]\label{ex2}
%\textcolor{red}{preferably, use inline equations here and use the notations I just introduced above}
Consider the DFA in Fig. \ref{fig:example}. We have that $\ccalR_{q_D^0}=\{q_D^0,q_D^1,q_D^F\}$. Also, we have that $\ccalX_{q_D^0\rightarrow q_D^0}=\ccalX\setminus(\ell_1\cup\ell_2)$, $\ccalX_{q_D^0\rightarrow q_D^1}=(\ccalX\setminus\ell_2)\cap\ell_1$, $\ccalX_{q_D^0\rightarrow q_D^F}=\ell_2\cap\ell_1$. Similarly, we have that $\ccalR_{q_D^1}=\{q_D^1,q_D^F\}$, $\ccalX_{q_D^1\rightarrow q_D^1}=\ccalX\setminus\ell_2$, and $\ccalX_{q_D^1\rightarrow q_D^F}=\ell_2$. 
\end{ex}

\subsection{Verifying Reach-Avoid Properties}\label{sec:VerReach}
\vspace{-0.1cm}
In what follows, we discuss how to check whether a transition from a DFA state $q_D$ to $q_D'\neq q_D$ can be enabled; later, we will discuss how this can be used to verify LTL properties. Specifically, we want to verify that given an initial set of system states associated with $q_D$, denoted by $\ccalX_{q_D}^0$, the previously discussed conditions (i)-(ii) can be satisfied; the detailed construction of $\ccalX_{q_D}^0$ will be discussed in Section \ref{sec:vltl}. 
Notice $\ccalX_{q_D\rightarrow q_D'}$ may contain more than one region of interest $\ell_i$. As discussed in Section \ref{sec:problem}, for each region of interest $\ell_i$ in $\ccalX_{q_D\rightarrow q_D'}$, the system selects the corresponding NN controller $\xi_i\in\Xi$. Hereafter, we collect all NN controllers associated with $\ccalX_{q_D\rightarrow q_D'}$ in a set denoted by $\Xi_{q_D\rightarrow q_D'}\subseteq\Xi$.\footnote{In case $\ccalX_{q_D\rightarrow q_D'}$ does not contain any region $\ell_i$, then $\Xi_{q_D\rightarrow q_D'}=\emptyset$ by definition of $\Xi$; see Ex. \ref{ex3}.}
In math, we want to show that there exists a finite horizon $H_{q_D}$ and at least one NN controller $\xi\in\Xi_{q_D\rightarrow q_D'}$, so that if the system evolves as per $\bbx_{t+1}=\bbf(\bbx_t,\xi)$ then the following two conditions hold for all possible initial system states in $\ccalX_{q_D}^0$: (i) $\bbx_{t} \in \ccalX_{q_D\rightarrow q_D}, \forall t\in[0,H_{q_D}-1]$ and (ii) $\bbx_{H_{q_D}}\in\ccalX_{q_D\rightarrow q_D'}$.
%
%\begin{equation}\label{eq:condI}
 %\bbx_{t} \in \ccalX_{q_D\rightarrow q_D}, \forall t\in[0,H_{q_D}-1] %\text{~and~} 
%\end{equation}
%\begin{equation}\label{eq:condII}
%\bbx_{H_{q_D}}\in\ccalX_{q_D\rightarrow q_D'}.
%\end{equation}
%for all possible initial system states $\bbx_0\in\ccalX_{q_D}^0$  and it evolves as per $\bbx_{t+1}=\bbf(\bbx_t,\xi)$.
%
If such a horizon $H_{q_D}$ and controller $\xi\in\Xi_{q_D\rightarrow q_D'}$ exist, then by definition of the set $\ccalX_{q_D\rightarrow q_D'}$ in \eqref{eq:X}, we have that within the time interval $[0,H_{q_D}-1]$, the transition from $q_D$ to $q_D$ (self-loop) is enabled, and at the time step $t=H_{q_D}$ the transition from $q_D$ to $q_D'$ occurs. In this case, we say that the DFA transition from $q_D$ to $q_D'$ is \textit{verified to be safe} when the system starts anywhere within $\ccalX_{q_D}^0$ and applies the NN controller $\xi$. %If for all available controllers $\xi\in\Xi_{q_D\rightarrow q_D'}$, there exists at least one initial state within $\ccalX_{q_D}^0$ for which either of the above conditions does not hold, then we say that the DFA transition is \textit{verified to be unsafe}.

To reason about safety of a DFA transition, we leverage existing reachability analysis tools that can compute forward reachable sets $\mathfrak{R}_t(\ccalX_{q_D}^0,\xi)$ collecting all possible states $\bbx$ that the system may reach after applying a feedback NN controller $\xi$ for $t$ time steps while starting anywhere in $\ccalX_{q_D}^0$. Given such reachable sets, it suffices to check if there exists a finite horizon $H_{q_D}$ and at least one controller $\xi\in\Xi_{q_D\rightarrow q_D'}$ such that the reachable sets satisfy the following two conditions: (i) $\mathfrak{R}_{t}(\ccalX_{q_D}^0,\xi)\subseteq \ccalX_{q_D\rightarrow q_D}, \forall {t}\in[0,\dots,H_{q_D}-1]$ and (ii) $\mathfrak{R}_{H_{q_D}}(\ccalX_{q_D}^0,\xi)\subseteq \ccalX_{q_D\rightarrow q_D'}.
$
%\begin{equation}\label{eq:reachi}
%\mathfrak{R}_{t}(\ccalX_{q_D}^0,\xi)\subseteq \ccalX_{q_D\rightarrow q_D}, \forall {t}\in[0,\dots,H_{q_D}-1]
%\end{equation}
%\begin{equation}\label{eq:reachii}
%\mathfrak{R}_{H_{q_D}}(\ccalX_{q_D}^0,\xi)\subseteq \ccalX_{q_D\rightarrow q_D'}.
%\end{equation}
If both conditions hold, we verify that the DFA transition from $q_D$ to $q_D'$ is \textit{safe} given the initial set of states $\ccalX_{q_D}^0$ and the controller $\xi$\cite{huang2019reachnn}. %In any other case, the DFA transition from $q_D$ to $q_D'$ is called \textit{unsafe}. 

Construction of exact reachable sets is computationally intractable.
%\footnote{\textcolor{red}{If you can add a sentence here briefly explaining why, that'd be great; read related papers for this}}. 
Thus, instead, we compute over-approximated reachable sets, denoted hereafter by  $\bar{\mathfrak{R}}_t(\ccalX_{q_D}^0,\xi)\subseteq\mathfrak{R}_t(\ccalX_{q_D}^0,\xi)$, using tools that can handle systems of the form (\ref{eq:dynamics}) with NN controllers \cite{Hu2020ReachSDPRA}. If the following two conditions are satisfied
\begin{equation}\label{eq:reachi2}
\bar{\mathfrak{R}}_{t}(\ccalX_{q_D}^0, \xi)\subseteq \ccalX_{q_D\rightarrow q_D}, \forall {t}\in[0,\dots,H_{q_D}-1]
\end{equation}
\begin{equation}\label{eq:reachii2}
\bar{\mathfrak{R}}_{H_{q_D}}(\ccalX_{q_D}^0,\xi)\subseteq \ccalX_{q_D\rightarrow q_D'},
\end{equation}
then we say that the considered DFA transition is verified to be safe given the initial set of states $\ccalX_{q_D}^0$ and a feedback NN controller $\xi$.\footnote{In practice, reachable sets over a large enough horizon $\bar{H}$ are computed. If there is not reachable set $\bar{\mathfrak{R}}_t$, for some $t\in\{0,\dots,\bar{H}\}$ that satisfies \eqref{eq:reachii2}, then we say that the system fails to reach this region of interest. This is in accordance with related works; see e.g., \cite{Hu2020ReachSDPRA, huang2019reachnn}.} Finally, if there is no self-loop for $q_D$, then $\ccalX_{q_D\rightarrow q_D}$ cannot be defined. In this case, such a transition from $q_D$ to $q_D'$ is verified to be safe if \eqref{eq:reachii2} holds for $H_{q_D}=1$.

%To account for the over-approximation error in the reachable sets, a DFA transition is called unsafe only if at least one of the following two conditions hold. The first condition requires the whole reachable set $\bar{\mathfrak{R}}_{t}(\ccalX_{q_D}^0,\xi)$ to be outside $\ccalX_{q_D\rightarrow q_D}$ (i.e., $\bar{\mathfrak{R}}_{t}(\ccalX_{q_D}^0,\xi)\cap\ccalX_{q_D\rightarrow q_D}=\emptyset$) for some $t\in[t_{q_D}^0,\dots,H_{q_D}-1]$ and for all $\xi\in\Xi_{q_D\rightarrow q_D'}$, i.e.,  The second condition requires the system to never reach the set $\ccalX_{q_D\rightarrow q_D}$, i.e., $\bar{\mathfrak{R}}_t(\ccalX_{q_D}^0,\xi)\cap \ccalX_{q_D\rightarrow q_D'}=\emptyset, \forall t\geq t_{q_D}^0.$ In any other case, we cannot reason about safety of the DFA transition due to the over-approximated reachable sets; for instance, consider a case where it holds that there exists a time instant $t$ such that  $\bar{\mathfrak{R}}_{t}(\ccalX_{q_D}^0,\xi)\cap\ccalX_{q_D\rightarrow q_D}\neq\emptyset$ and $\bar{\mathfrak{R}}_{t}(\ccalX_{q_D}^0,\xi)\cap\{\ccalX\setminus\ccalX_{q_D\rightarrow q_D}\}\neq\emptyset$.\footnote{\textcolor{red}{Jun, as we discussed last time, please check the literature to see if this terminology is consistent with other papers}} 

%\vspace{-0.1cm}
\begin{ex}[Verifying Reach-Avoid Properties (cont)]\label{ex3}
%\textcolor{red}{preferably, use inline equations here and use the notations I just introduced above}
Consider the DFA in Fig. \ref{fig:example} and, specifically, the transition from $q_D^0$ to $q_D^1$. To reach $q_D^1$ from $q_D^0$, the available controllers are $\Xi_{q_D^0\rightarrow q_D^1}=\{\xi_1\}$ by construction of the set $\ccalX_{q_D^0\rightarrow q_D^1}$; see Ex. \ref{ex2} and Section \ref{sec:problem}. Let $\ccalX_{q_D^0}^0=\ccalX_0$. Then, we compute reachable sets $\bar{\mathfrak{R}}_t(\ccalX_0,\xi_1)$; see Fig. \ref{fig:reach}. Observe that $H_{q_D^0}=3$, since the sets $\bar{\mathfrak{R}}_t(\ccalX_0,\xi_1)$ for $t=0,1,2$ satisfy \eqref{eq:reachi2} and $\bar{\mathfrak{R}}_3(\ccalX_0,\xi_1)$ satisfies \eqref{eq:reachii2}. Thus, the transition from $q_D^0$ to $q_D^1$ is verified to be safe. For the transition from $q_D^1$ to $q_D^F$, we have that $\Xi_{q_D^1\rightarrow q_D^F}=\{\xi_2\}$. Verification of this transition %requires the computation of reachable sets $\bar{\mathfrak{R}}_t(\ccalX_{q_D^1}^0,\xi_2)$ shown in Fig. \ref{fig:reach}; construction of $\ccalX_{q_D^1}^0$ 
will be discussed in Ex. \ref{ex4}. Finally, we have $\Xi_{q_D^0\rightarrow q_D^0}=\emptyset$ and $\Xi_{q_D^1\rightarrow q_D^1}=\{\xi_1\}$.
\end{ex}
%\begin{rem}[Reachability Analysis]

%\end{rem}
%Hence some techniques alter the goal to find the outer-approximation set and serve as the superset of the exact reachable set: $\bar{\mathfrak{R}}_t(\ccalX_{q_D})\supseteq \mathfrak{R}_t(\ccalX_{q_D})$:
%\begin{equation}\label{eq:reachi}
%\bar{\mathfrak{R}}_{t'}(\ccalX_{q_D}^0)\subseteq \ccalX_{q_D\rightarrow q_D}, %\forall {t'}\in[t,\dots,H-1]
%\end{equation}
%\begin{equation}\label{eq:reachii}
%\bar{\mathfrak{R}}_H(\ccalX_{q_D}^0)\subseteq \ccalX_{q_D\rightarrow q_D'}.
%\end{equation}

%Multiple reachability analysis tools have been proposed \textcolor{red}{[ADD REF]}, our proposed framework is applicable to any tool. In this paper, we choose the Reach-SDP as our toolbox.

%\begin{equation}\label{eq:UnsafeI}
%\mathfrak{R}_{t}(\ccalX_{q_D}^0)\cap\ccalX_{q_D\rightarrow q_D}=\emptyset, \text{for any} {t}\in[0,\dots,H_{q_D}-1]
%\end{equation}
%\begin{equation}\label{eq:UnsafeII}
%\mathfrak{R}_H(\ccalX_{q_D}^0)\subseteq \ccalX_{q_D\rightarrow q_D'}.
%\end{equation}

\vspace{-0.1cm}
\subsection{From Verification of Reach-Avoid Properties to Verification of LTL formulas}\label{sec:vltl}
\vspace{-0.1cm}

To verify that the system satisfies $\phi$, it suffices to check that the final DFA state can be reached from the initial state by enabling a sequence of  DFA transitions that are verified to be safe; see also Section \ref{sec:DFA}. To check this we rely on applying graph-search methods over the DFA state while verifying on-the-fly safety of DFA transitions using reachability analysis; see Section \ref{sec:VerReach}. %\footnote{\textcolor{red}{once the algorithm is done, please add references to its lines}}
Specifically, first we view the DFA as a directed graph $D=\{\ccalV, \ccalE\}$ with vertices $\ccalV$ and edges $\ccalE$ that are determined by the set of states and transitions of the DFA. As discussed in Section \ref{sec:VerReach}, to verify safety of DFA transition, an initial set of states is needed denoted by $\ccalX_{q_D}^0$. This set captures all possible states in $\ccalX$ that the system may have when it reaches a DFA state $q_D$. As a result, $\ccalX_{q_D}^0$ depends on the previous DFA states that the system has gone through to reach $q_D$. To simplify the proposed algorithm, we pre-process $D$ so that each node in $D$ that can be reached through multiple paths (excluding self-loops) originating from $q_D^0$ is replicated so that each replica can be reached through a unique path (excluding self-loops). %\footnote{Note we exclude self-loops since the system is}
For each vertex $q_D\in \ccalV$, we define sets that collect its incoming and outgoing edges denoted by $\ccalE_{q_D}^{\text{in}}$ and $\ccalE_{q_D}^{\text{out}}$, respectively.
If the number of incoming edges (excluding self-loops) for $q_D$ is greater than $1$, i.e., $|\ccalE_{q_D}^{\text{in}}|>1$ we create $|\ccalE_{q_D}^{\text{in}}|$ copies of $q_D$ denoted by $q_D^i$. Each node $q_D^i$ has only one incoming edge which is selected to be the $i$-th edge in $\ccalE_{q_D}^{\text{in}}$, denoted by $\ccalE_{q_D}^{\text{in}}(i)$, while its outgoing edges remain the same as in the original node $q_D$. Then we add all copies to the graph and remove the original nodes $q_D$. We denote the resulting graph by $D'$. %The output of Alg. \ref{algo2_pre} is the processed graph denoted by $D'$. 

%\footnotesize
\begin{algorithm}[t]
\footnotesize
\caption{Reach\_DFS Algorithm}\label{algo2}
\begin{algorithmic}[1]
\State \textbf{Input}: $D'$; NN controllers $\Xi$; System dynamics \eqref{eq:dynamics}
\State \textbf{Output}: Verification output  $R\in\{[\texttt{True},\boldsymbol\xi],\texttt{False}\}$
\State Initialize $q_D^{\text{cur}} \leftarrow q_D^0$; $\ccalX_{q_D^{\text{cur}}}^0=\ccalX_0$;  $\bbq=q_D^{\text{cur}}$; $\ccalV_{\text{vis}}=\{q_D^{\text{cur}}\}$;  $g(q_D^{\text{cur}})\leftarrow \ccalX_{q_D^{\text{cur}}}^0$;  $\boldsymbol\xi \leftarrow \emptyset$; $E=\texttt{False}$;
\While{$E \neq \texttt{True}$}
\State \text{Randomly select} $q_D^{\text{next}}\in \ccalR_{q_D^{\text{cur}}}$; $\ccalV_{\text{vis}} \leftarrow \ccalV_{\text{vis}} \cup q_D^{\text{next}}$
\State Compute set of available controllers $\Xi_{q_D^{\text{cur}}\rightarrow q_D^{\text{cur}}}$\; 
\If {$\exists$ $H_{q_D^{\text{cur}}}$ and $\xi\in\Xi_{q_D^{\text{cur}}\rightarrow q_D^{\text{next}}}$ for \eqref{eq:reachi2}-\eqref{eq:reachii2}}
\State $\bbq=\bbq|q_D^{\text{cur}}$; $\boldsymbol\xi\leftarrow \boldsymbol\xi| \xi$; $\ccalX_{q_D^{\text{next}}}^0\leftarrow \bar{\mathfrak{R}}_{H_{q_D^{\text{cur}}}}(q_D^{\text{cur}})$;
\State   $q_D^{\text{cur}} \leftarrow q_D^{\text{next}}$;
$g(q_D^{\text{next}})\leftarrow \ccalX_{q_D^{\text{next}}}^0$;
\If {$q_D^{\text{cur}}=q_D^F$}
\State $R = [\texttt{True},\boldsymbol\xi]$;  $E= \texttt{True}$;
\EndIf
\Else
\State $\ccalR_{q_D^{\text{cur}}} \leftarrow \ccalR_{q_D^{\text{cur}}} \setminus q_D^{\text{next}}$
\While {$\ccalR_{q_D^{\text{cur}}}\setminus \ccalV_{\text{vis}}  = \emptyset \wedge \bbq \neq \emptyset$}
\State $q_D^{\text{cur}} \leftarrow \bbq(\text{end})$; $\ccalX_{q_D^{\text{cur}}}^0\leftarrow g(q_D^{\text{cur}})$;
\State $\bbq\leftarrow \bbq \setminus \bbq(\text{end})$; $\boldsymbol\xi\leftarrow \boldsymbol\xi\setminus \boldsymbol\xi(\text{end})$
\EndWhile
\If{$\ccalR_{q_D^{\text{cur}}}\setminus \ccalV_{\text{vis}}  = \emptyset\wedge \bbq = \emptyset $}
\State $R\leftarrow \texttt{False}$; $E\leftarrow \texttt{True}$;
\EndIf
%\State Randomly Select $q_D^{\text{next}} \in \ccalR_{q_D^{\text{cur}}}$;
\EndIf
\EndWhile
\end{algorithmic}
\end{algorithm}
\normalsize

Next, we apply a Depth-first search (DFS) method over  $D'$ to see if $q_D^F$ can be reached from $q_D^0$ through a sequence of DFA transitions that are verified to be safe. This process is summarized in Alg. \ref{algo2}.
%To check if this is possible, we view the pruned DFA as a directed graph and we apply a Depth-first search (DFS) method over it. this process is summarized in Alg. \ref{algo2}
%
The inputs to this algorithm are the graph $D'$, the set $\Xi$ of NN controllers and the system dynamics \eqref{eq:dynamics} (line 1). In what follows, we denote by $q_D^{\text{cur}}$ the currently visited node in $D'$. The set of all possible states in $\ccalX$ that the system can be when it reaches $q_D^{\text{cur}}$ is denoted by $\ccalX_{q_D^{\text{cur}}}^0$. We initialize $q_D^{\text{cur}}$ as $q_D^{\text{cur}}=q_D^0$ and $\ccalX_{q_D^{\text{cur}}}^0=\ccalX_0$. Also, we define a set $\ccalV_{\text{vis}}$ that collects all nodes in $D'$ that have been visited and a sequence $\bbq$ of nodes that points to the current path from $q_D^0$ towards $q_F$. They are initialized as $\ccalV_{\text{vis}}=\{q_D^0\}$ and $\bbq=q_D^0$. Also we initialize the control strategy $\boldsymbol\xi$ as an empty sequence. We also define a function $g:\ccalV\rightarrow \ccalX$ that maps a DFA state $q_D\in\ccalV$ to the corresponding set $\ccalX_{q_D}^0$; this function that is constructed on-the-fly is need only as way store and recover from memory the sets $\ccalX_{q_D}^0$ (lines 2-3). 
Given  $q_D^{\text{cur}}$, we randomly select a next state $q_{\text{next}}\in \ccalR_{q_D^{\text{cur}}}$. Then we apply reachability analysis over the transition from $q_D^{\text{cur}}$ to $q_D^{\text{next}}$ using the NN controllers $\Xi_{q_D^{\text{cur}}\rightarrow q_D^{\text{next}}}\subseteq\Xi$ and the initial set of states $\ccalX_{q_D^{\text{cur}}}^0$ (lines 5-6); see Section \ref{sec:VerReach}. 
If the transition is verified to be safe, then we append $q_D^{\text{cur}}$ to $\bbq$. Also, we append the controller $\xi\in\Xi_{q_D^{\text{cur}}\rightarrow q_D^{\text{next}}}$ for which this transition is safe to $\boldsymbol\xi$ that denotes the current control strategy to reach $q_D^{\text{next}}$ from $q_D^0$ (lines 7-8). The corresponding horizon $H_{q_D^\text{cur}}$ should be stored as well; we abstain from this for simplicity of presentation.
%Set $\bbq$ will collect $q_D^{\text{cur}}$ to record possible feasible path from $q_D^0$ to $q_D^F$. 
The final reachable set $\bar{\mathfrak{R}}_{H_{q_D^{\text{cur}}}}(q_D^{\text{cur}},\xi)$ becomes the set $\ccalX_{q_D^{\text{next}}}^0$; see Ex. \ref{ex4} and Fig. \ref{fig:reach} as well (line 8). Also, $g(q_D^{\text{next}})$ is constructed on-the-fly as $g(q_D^{\text{next}})=\ccalX_{q_D^{\text{next}}}^0$ and we replace $q_D^{\text{cur}}$ with $q_D^{\text{next}}$ (line 9).
If the transition $q_D^{\text{cur}}\rightarrow q_D^{\text{next}}$ is not verified to be safe, we remove $q_D^{\text{next}}$ from the set $\ccalR_{q_D^{\text{cur}}}$  (see \eqref{eq:R}). Then we keep taking out the last element in $\bbq$ and $\boldsymbol\xi$, denoted by $\bbq(\text{end})$ and $\boldsymbol\xi(\text{end})$, while $\bbq(\text{end})$ is assigned to $q_D^{\text{cur}}$ until we find another state in $\ccalR_{q_D^{\text{cur}}}$ that has not been not visited yet (lines 12-16). 
%\footnote{\textcolor{red}{This is not clear. When do you stop? I.e., what does $q_D^{\text{cur}}$ have to satisfy to stop this? Explain this in words.}} 
The above process is repeated until $q_D^{\text{cur}}$ is updated to be $q_D^F$. In this case, we have found a NN control strategy $\boldsymbol\xi$ that satisfies $\phi$ for all initial states $\bbx_0\in\ccalX_0$ (line 10-11). If $\bbq$ is empty yet we fail to find another state in $\ccalR_{q_D^{\text{cur}}}$ that is not visited, then the proposed method cannot find a feasible path from $q_D^0$ to $q_D^F$ (even though it may exist; see Section \ref{sec:complete}) (lines 17-18).
We note that any other graph-search method in conjunction with reachability analysis can be used as well.

\begin{ex}[LTL Verification (cont)]\label{ex4}
%\textcolor{red}{preferably, use inline equations here and use the notations I just introduced above}
%\textcolor{red}{The text should be updated based on the Figure \ref{fig:reach}; see all previous examples}
We continue Ex. \ref{ex3}; see also Fig. \ref{fig:reach}. To verify the DFA transition from $q_D^{\text{cur}}=q_D^1$ to $q_D^{\text{next}}=q_D^F$, we initialize $\ccalX_{q_D^1}^0=\bar{\mathfrak{R}}_3(\ccalX_0,\xi_1)$. 
By computing $\bar{\mathfrak{R}}_t(\ccalX_{q_D^1}^0,\xi_2)$, we verify that this DFA transition is safe. Thus, there exists $\boldsymbol\xi=\xi_1,\xi_2$ where $\xi_1$ and $\xi_2$ are applied for $3$ and $2$ time steps so that $\bbf_{\xi}\models\phi$, for all $\bbx_0\in\ccalX_0$. 
\end{ex}
\vspace{-0.2cm}
\subsection{Correctness \& Completeness}\label{sec:complete}
\vspace{-0.1cm}
%In the following proposition, we show correctness of Algorithm. \ref{algo1}. %, i.e., if there exists a NN control strategy $\boldsymbol\xi$ as defined in Problems \ref{pr1}-\ref{pr2}, then Alg. \ref{algo1} will find it.
\begin{prop}[Correctness]
% Alg. \ref{algo1} 
The proposed method is correct, i.e., the computed $\boldsymbol\xi$ solves Problem \ref{pr1}. 
\end{prop}
\begin{proof}
%This result holds by construction of Alg. \ref{algo1}. Particularly, Alg. \ref{algo1} 
This result holds by construction of the proposed method. Particularly, we
ensure that the closed-loop system \eqref{eq:closedloop} driven by $\boldsymbol\xi=\mu(0),\dots,\mu(K)$, where each $\mu(k)$ is applied for $H_k$ time-steps, generates trajectories $\tau(\bbx_0)=\bbx_0,\dots,\bbx_F$, where $F=\sum_{t=0}^{H_k}$, that satisfy the system dynamics and the LTL formula $\phi$, $\forall \bbx_0\in\ccalX_0$. %This can be shown by induction. At $k=0$, the current DFA state is $q_D^0$ and the system applies  $\mu(0)$ for $H_0$ time steps. By construction of Alg. \ref{algo2} and due to the deterministic nature of the DFA, it holds that a unique DFA state will be reached after $H_{0}$ steps, denoted by $q_D^1$. Assume that at time $t'=\sum_{}$
%The controllers $\mu(k)$ are applied sequentially, where each $\mu(k)$ is applied for $H_k$ time steps. By construction of $\xi$, once $\mu(k)$ is applied for $H_k$ time steps a new DFA state will be reached, for any possible initial system state $\bbx_0\in\ccalX_0$. Then the controller $\mu(k+1)$ will be applied for $H_{k+1}$ steps until at $k=G$, the final automaton states.
%
%By construction of $\boldsymbol\xi$, once the system reaches a DFA state $q_D$, it selects
%the NN controller $\mu(k)$ is applied when the system reaches for a DFA state $q_D$
\end{proof}
\begin{rem}[Computational Efficiency vs Completeness]\label{rem:complete}
In general, our method is not complete in the sense that it may not find a control strategy $\boldsymbol\xi$ that satisfies $\phi$, for all $\bbx_0\in\ccalX_0$, even though such a strategy exists. This is due to the fact that (a) Alg. \ref{algo2} computes over-approximated reachable sets and that (b) it does not exhaustively search over all possible combinations of NN control actions that the system can apply when  a new DFA state is reached. 
As for (b), for instance, \textit{given a initial set of states for $q_D^{\text{cur}}$}, Alg. \ref{algo2} reasons about safety of a transition from $q_D^{\text{cur}}$ to $q_D^{\text{next}}$, using only controllers selected from $\Xi_{q_D^{\text{cur}}\rightarrow q_D^{\text{next}}}$. If this transition is unsafe, then it is discarded. However, this transition may become feasible if the initial set of states changes, which can happen by applying a controller from $\Xi_{q_D^{\text{cur}}\rightarrow q_D^{\text{cur}}}$ for some $\hat{H}$ time steps.
%However, it is possible that if the system applied a controller from  $\Xi_{q_D^{\text{cur}}\rightarrow q_D^{\text{cur}}}$ for some $\hat{H}>0$ time steps, so that it stays in $q_D^{\text{cur}}$, then the DFA transition from $q_D^{\text{cur}}$ to $q_D^{\text{next}}$ may switch to `safe' using controllers from $\Xi_{q_D^{\text{cur}}\rightarrow q_D^{\text{next}}}$. The reason is that in this case a new set of initial states is considered for verifying this DFA transition. By following such an approach, subsequent DFA transitions may switch from safe to unsafe and vice versa depending on $\hat{H}$. 
%Nevertheless, searching over all possible values of such horizons $\hat{H}$, for all DFA states, is computationally intractable. 
Additionally, as soon as Alg. \ref{algo2} finds a $\xi\in\Xi_{q_D^{\text{cur}}\rightarrow q_D^{\text{next}}}$ for which the corresponding transition is safe, it proceeds to new DFA transitions. However, that transition may be safe for other controllers in $\xi\in\Xi_{q_D^{\text{cur}}\rightarrow q_D^{\text{next}}}$ as well, where each one yields a different initial set for subsequent DFA transitions affecting their safety.
%as  a DFA transition is safe for some $\xi\in\Xi_{q_D^{\text{cur}}\rightarrow q_D^{\text{next}}}$, Alg. \ref{algo2} proceeds by considering only this $\xi_i$ and the corresponding initial set for $q_D^{\text{next}}$ which may affect safety of subsequent DFA transitions as well. 
%Alg. \ref{algo1} 
The proposed method can be extended to account for these additional control actions at the expense of increasing its computational cost. We note that such trade-offs are quite common in related works; see e.g., \cite{leahy2022fast}. % Alg. \ref{algo1} 
The proposed method is complete if (i) there are no self-loops in the DFA states, or if $\Xi_{q_D\rightarrow q_D}=\emptyset$ for all states (see e.g., Ex. \ref{ex3}); (ii) $|\Xi_{q_D\rightarrow q_D'}|=1$, for all $q_D\in\ccalQ_D\setminus\{q_D^F\}$; (iii) the reachable sets are accurately computed. %For instance, for the DFA in Fig. \ref{fig:example}, we have that  $\Xi_{q_D^0\rightarrow q_D^0}=\emptyset$ and $\Xi_{q_D^1\rightarrow q_D^1}=\{\xi_1\}$.
%that it can apply when in the currently visited DFA state $q_D^{\text{cur}}$, examines only the controllers $\Xi_{q_D^{\text{cur}}\rightarrow q_D^{\text{next}}}$ to reason about safety of a DFA transition from $q_D^{\text{cur}}$ to $q_D^{\text{next}}$
\end{rem}

%\textcolor{red}{[I need to add examples/figures to illustrate the above process - verifying single transition]}

\vspace{-0.4cm}
\section{Experiments}\label{sec:sim}
\vspace{-0.2cm}
%For both case studies, we train separate NN controllers for different regions of interest such that each time a robot receive the requirement from a sub-task, corresponding NN controller will be selected for the verification process.

%\textcolor{red}{[you also need to explain how you generate data]}
% \subsection{Double Integrator}
% We consider a double integrator system:

% \begin{equation}
% x_{t+1} = \underbrace{
% \begin{bmatrix}
% -0.5&0\\
% 0.1&-0.2
% \end{bmatrix}
% }_A x_t + \underbrace{
% \begin{bmatrix}
% 1&0\\
% 0&1
% \end{bmatrix}
% }_B u_t
% \end{equation}

% \begin{equation}
% u_t = -\underbrace{
% \begin{bmatrix}
% 0.3&0.3\\
% 0.2&0.2
% \end{bmatrix}
% }_{K_1} x_t + K_2 r
% \end{equation}

% with $r$ denoted as reference input of the system. Also we have $x_t\in \mathcal{R}^2$ and $u_t\in \mathcal{R}^2$. Let the equilibrium state of the system $X_{\text{EQL}}=r$ where $x_{\text{EQL}} = -(A-BK_1-I)^{-1}BK_2r$. Then we have $K_2 = -B^{-1}(A-BK_1-I)$. 

%\subsection{6D Quadrotor}

In this section, we demonstrate our framework on an unmanned aerial vehicle (UAV) with various LTL properties. %All simulations have been run on a computer with 16 GB RAM and GeForce RTX 3080 graphic card.%...\textcolor{red}{[briefly mention specifications-processor, ram, etc]}

\textbf{System Dynamics:} 
We consider a UAV with dynamics as in (\ref{eq:dynamics}) where the matrices $\bbA_t, \bbB_t,$ and $\bbc_t$ are defined as in \cite{Hu2020ReachSDPRA}.
% We consider a UAV with discrete-time dynamics of the form 
% $\bbx_{t+1}=\bbA\bbx_t+\bbB\bbu_t+\bbc$ operating in a  workspace $\Omega\subseteq\mathbb{R}^3$; definition of the matrices $\bbA, \bbB$, and $\bbc$ as well as the input constraints $\ccalU_t$ can be found in \cite{Hu2020ReachSDPRA}. 
%described in \cite{ARCH19_Verification_of_Closed_loop_Systems} and rewrite the nonlinear dynamics into \eqref{eq:quad}. The original nonlinear dynamics in \cite{ARCH19_Verification_of_Closed_loop_Systems} is discretized with a sampling rate $t_s=0.3s$ with Runge-Kutta 4th order method and applied as the prediction model in MPC.
%\footnote{\textcolor{red}{What do you mean by prediction model? I believe you wan to say that you have a continuous model borrowed from [31] which you have discretized using this sampling rate. This should be clear (if that's what you mean). If that's also the case then this should be discussed earlier, when we mention the dynamics}} 
The UAV state is defined as $\bbx=[p_x;p_y;p_z;v_x;v_y;v_z]\in\mathbb{R}^6$ capturing the position and velocity. %with constraint $\bbx \in [-5,5]\times[-5,5]\times[-5,5]\times[-1,1]\times[-1,1]\times[-1,1]$.
The control input $\bbu\in\mathbb{R}^3$ is a function of pitch, roll and thrust. 
%with constraints $\bbu \in [-\tan\frac{\pi}{9},\tan\frac{\pi}{9}]\times[-\tan\frac{\pi}{9},\tan\frac{\pi}{9}]\times[0,2g]$. %\footnote{\textcolor{red}{this notation is unclear. I would write separate constraints for each component}}.%\footnote{\textcolor{red}{What is $A^C$? Let's not introduce new notation in this section unless it is necessary}}.
%

\textbf{NN controllers:} In what follows, we define LTL specifications that require the UAV to visit certain disjoint regions of interest $\ell_i\in\mathbb{R}^3$. We highlight that the the regions $\ell_i$ are defined only over the UAV position, i.e., $\ell_i \subseteq\Omega$. Also, since we assume disjoint regions, the DFA can be pruned by removing infeasible DFA transitions reducing the computational cost for verification \cite{kantaros2018text}.  %\footnote{\textcolor{red}{Also, we need more details about how trained the NN. I am adding some but please fill in the details. Also, double check if what I have written is accurate}}. 
We train the NN controllers similarly to \cite{Hu2020ReachSDPRA}; nevertheless, any other method (e.g., RL) can be used to train them. Specifically, to train $\xi_i$, we leverage nonlinear Model Predictive Control (MPC) methods. First we generate a random set of states $\bbx\in\mathbb{R}^6$. Starting from each one of these states, we generate a sequence of pairs of states and control inputs that drive the UAV towards the interior of $\ell_i$ using an off-the-shelf MPC solver \cite{andrei2017sqp}. Each pair constitutes a data point in a training dataset. Using this dataset we train feedforward NN controllers $\xi_i$ with $2$ hidden layers, $30$ neurons/layer,  and ReLU activation functions. %\textcolor{red}{We emphasize that }% Our goal is to check whether given/trained NN controllers $\xi_i$ can be composed to satisfy an LTL task.} %The NNs have $2$ hidden layers and $30$ neurons in each layer. %Adam optimizer in PyTorch
%\footnote{\textcolor{red}{I believe it is the same structure across all case studies. Later in each case study mention the size of the training dataset as well. I believe this size is different across your case studies. Also, I still think we need to report some sort of accuracy on the training set for each NN in each case study.}} 
%The average training time of a NN $\xi_i$ across all case studies was 1 hour approximately.

%A nonlinear MPC using SNOPT [ADD REF] is implemented to generate data points to train different NN controllers. Neural network controllers were trained with ReLU activation function and Adam optimizer in PyTorch.

\begin{figure}[t]
%\footnotesize
\centering
\begin{subfigure}[t]{0.47\linewidth}
\centering
    \includegraphics[width=\linewidth]{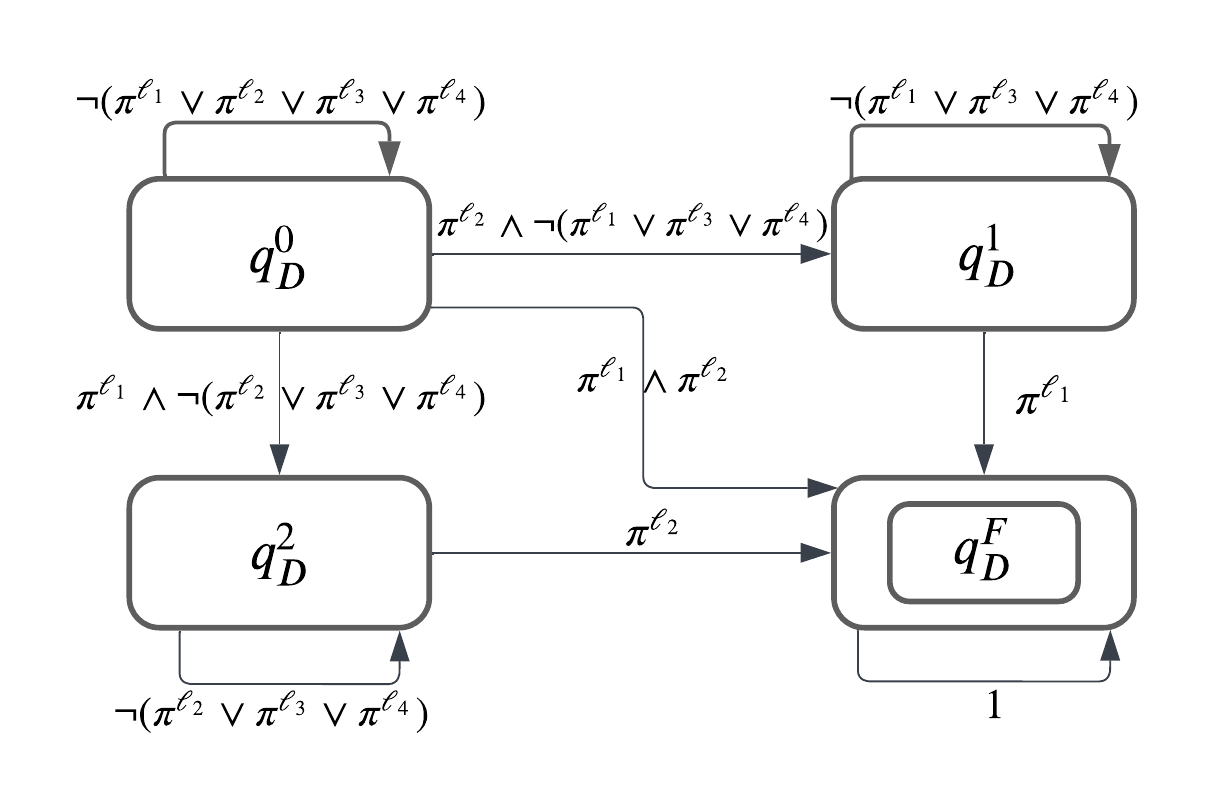}
    \caption{Case Study 1}
    \label{fig:first}
\end{subfigure}
\hspace{2em}%
\begin{subfigure}[t]{0.26\linewidth}
\centering
    \includegraphics[width=\linewidth]{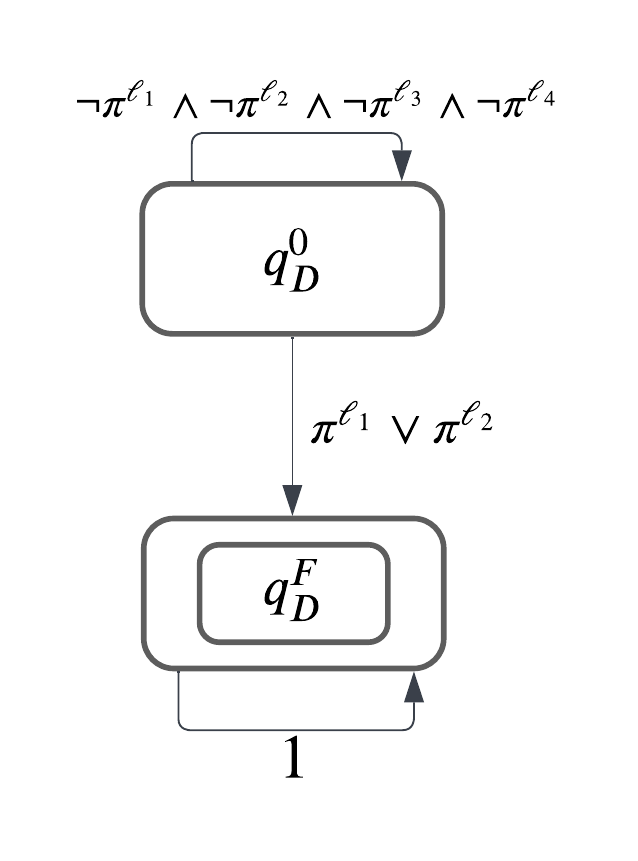}
    \caption{Case Study 2}
    \label{fig:second}
\end{subfigure}
\caption{DFA for the LTL formulas in case studies I \& II.}
\label{fig:case1}\vspace{-0.5cm}
\end{figure}
\normalsize

\textbf{Case Study I:} In this case study, we consider the following LTL formula $\phi=(\Diamond\pi^{\ell_1})\wedge (\Diamond\pi^{\ell_2}) \wedge (\neg (\pi^{\ell_3} \vee \pi^{\ell_4} )\ccalU \pi^{\ell_1}) \wedge (\neg (\pi^{\ell_3} \vee \pi^{\ell_4} )\ccalU \pi^{\ell_2})$ requiring the UAV to eventually visit the regions $\ell_1$ and $\ell_2$, in any order, while avoiding the obstacles $\ell_3$ and $\ell_4$. %The coordinates of these regions are..\footnote{\textcolor{red}{Let's add the following only if we have space}}
%$\ell_1 = [0.5, 1.5]\times[0.5, 1.5]\times[0.5, 1.5]$, $\ell_2=[1.5, 2.5]\times [1.5, 2.5]\times[ 1.5, 2.5]$, $\ell_4=[1.5, 1.7]\times[ 0.9, 1.1]\times[ 0, 1.8]$, $\ell_5 = [1.6, 1.8]\times[1.9, 2.0]\times[ 0, 1.8]$; while (b) has $\ell_5=[1.4, 1.7]\times[1.5, 1.8]\times[0, 1.8]$ w
%Figure ?? shows the location of these regions projected on the $x-y$ plane.%\footnote{\textcolor{red}{just in case we add such a figure}}. 
This formula corresponds to the DFA shown in Fig. \ref{fig:first}; notice that the transition from $q_D^0$ to $q_D^{F}$ was pruned and, therefore, not considered during verification since it requires the UAV to be present in more than one region simultaneously. %\footnote{\textcolor{red}{you need to say which one is pruned. Also, please re-draw the automaton. All transitions starting from the initial state seem to be infeasible which is wrong. Please be careful with these things so that we avoid doing the same work twice. Also, it looks a bit sloppy to use regions $\ell_1,\ell_2,\ell_3,\ell_5$. What happened to $\ell_3$? I would suggest you rename them as $\ell_1,\ell_2,\ell_3,\ell_4$. We can fix this at the end. Not a big priority for now}} %\footnote{\textcolor{red}{check the automaton on that online tool}}
%
%\footnote{\textcolor{red}{ I would like to see here their accuracy on their training dataset if possible (ignoring obstacles)}}. 
%
To train $\xi_1$ and $\xi_2$ we collected $16000$ datapoints for each controller. Given these trained NN controllers and an initial set $\ccalX_0$ of states defined as an ellipsoid centered at $[3.5;3.5;2.9]$ with shape matrix of $\text{diag}[0.5^2;0.5^2;0.5^2]$.
%\footnote{\textcolor{red}{the center is not enough to define an ellipsoid. What about the radii? We can fix this at the end. Not a big priority for now}}%\footnote{\textcolor{red}{It is not clear what these numbers mean with respect to an ellipsoid. You should explain this.}}
We check if there exists a sequence of control actions that satisfy $\phi$. Particularly,  first it investigates the DFA transition from $q_D^0$ to $q_D^1$. This transition requires the UAV to stay within the obstacle-free space (i.e., avoid the obstacles $\ell_3$ and $\ell_4$) and  eventually reach $\ell_2$. The corresponding reachable sets for this DFA transition are shown in Figure \ref{fig:c1first}. Notice that the reachable sets $\bar{\mathfrak{R}}_t(\ccalX_0,\xi_2)$ are fully outside the obstacle regions for $t=1,2$ while at $t=3$ the corresponding reachable set is fully inside $\ell_2$. Thus, this DFA transition is verified to be safe. Next, the DFA transition from $q_D^1$ to $q_D^F$ is considered requiring the robot to reach $\ell_1$ while avoiding the obstacle regions. The set of initial states for this transition is $\ccalX_{q_D^1}^0=\bar{\mathfrak{R}}_3(\ccalX_0,\xi_2)$. After computing reachable sets $\bar{\mathfrak{R}}_t(\ccalX_{q_D^1}^0,\xi_1)$ (not shown), we can see that for $t=1,2$ they are outside the obstacles while $\bar{\mathfrak{R}}_3(\ccalX_{q_D^1}^0,\xi_1)$ is inside $\ell_1$ verifying that this DFA transition is also safe. Thus, there exists a control strategy $\boldsymbol\xi=\xi_2,\xi_1$ where $H_1=H_2=3$ such that $\bbf_{\xi}\models\phi$ for all $\bbx_0\in\ccalX_0$.

\begin{figure}[t]
%\footnotesize
\centering
\begin{subfigure}[t]{0.45\linewidth}
    \includegraphics[width=\linewidth]{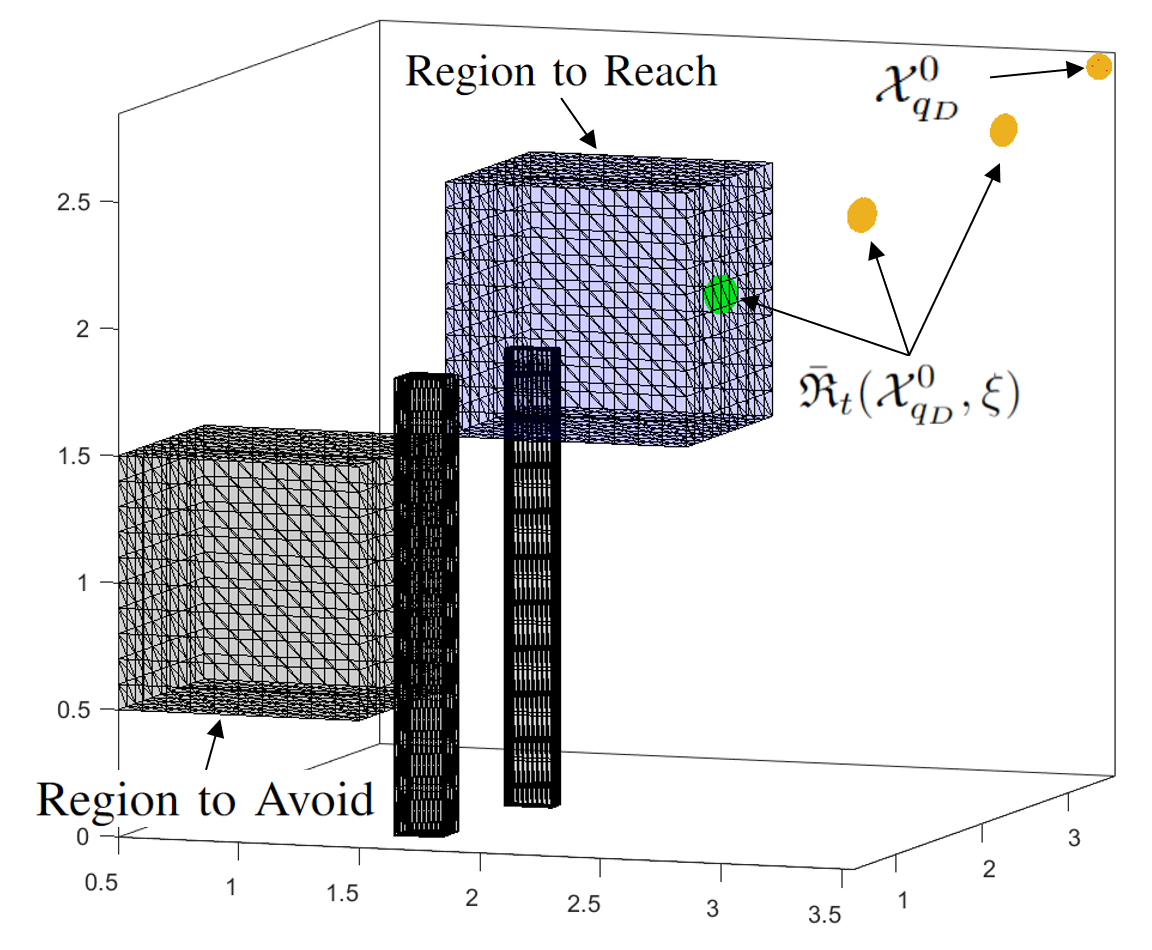}%\vspace{-1mm}
    % \caption{$\neg\pi^{\ell_1}\wedge\pi^{\ell_2}\wedge\neg\pi^{\ell_4}\wedge\pi^{\ell_5}$}
    \caption{Case I: $q_D^0\rightarrow q_D^1$}\vspace{-0.1cm}
    \label{fig:c1first}
\end{subfigure}
\hspace{2em}%
\begin{subfigure}[t]{0.38\linewidth}
        \includegraphics[width=\linewidth]{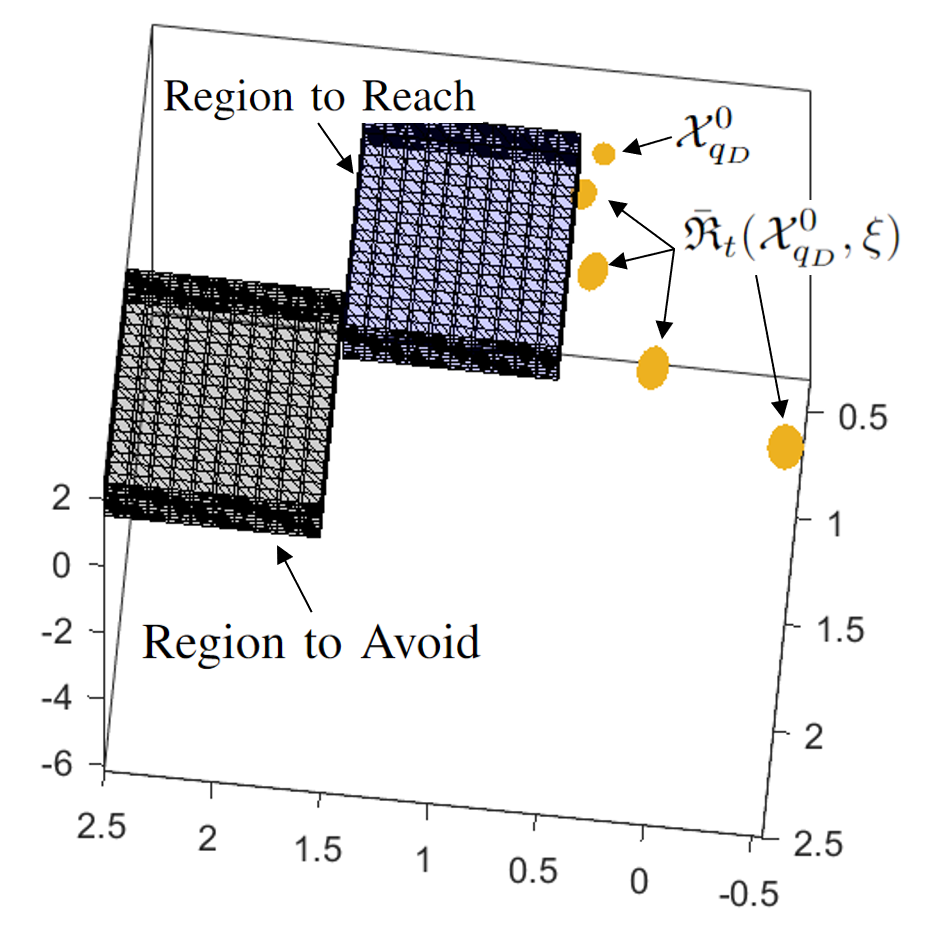}%\vspace{-1mm}
    \caption{Case II: $q_D^0\rightarrow q_D^F$}\vspace{-0.1cm}
    \label{fig:c2first}
\end{subfigure}
\caption{Verification of a DFA transition for the case studies I and II. 
The regions of interest that need to be avoided/reached to enable this DFA transition are shown with black/blue colour. All other regions are not shown. All reachable sets that are outside/inside the regions to reach are shown with orange/green color.} %If a reachable set is inside the region to reach, it is illustrated with green color.}
\label{fig:case1}\vspace{-0.8cm}
\end{figure}
\normalsize

\textbf{Case Study II:} In this case study, we consider an LTL formula $\phi = \Diamond (\pi^{\ell_1} \vee \pi^{\ell_2})\wedge ( \neg (\pi^{\ell_3} \vee \pi^{\ell_4}) \ccalU (\pi^{\ell_1} \vee \pi^{\ell_2}))$ requiring the UAV to eventually visit either region $\ell_1$ or $\ell_2$, while avoiding obstacles $\ell_3$ and $\ell_4$. This formula corresponds to the DFA shown in Figure \ref{fig:second}. The NN controller $\xi_2$ is the same as in the previous case study while $\xi_1$ was trained using only $100$ datapoints to mimic a poorly designed controller. The initial set of states $\ccalX_0$ is defined as an ellipsoid centered at $[0.4;0.4;0.3]$ with shape matrix of $\text{diag}[0.5^2;0.5^2;0.5^2]$.
Given the DFA, we investigate the transition from $q_D^0$ to $q_D^F$ which requires the UAV to reach either $\ell_1$ or $\ell_2$ while avoiding both $\ell_3$ and $\ell_4$. In other words, we have that $\ccalX_{q_D^0\rightarrow q_D^0}=\{\Omega\setminus(\ell_3\cup\ell_4)\}$ and $\ccalX_{q_D^0\rightarrow q_D^F}=\{\ell_1\cup\ell_2\}$. Thus, we have that $\Xi_{q_D^0\rightarrow q_D^F}=\{\xi_1,\xi_2\}$. First, we check if this DFA transition is safe using  $\xi_1$. The generated reachable sets are shown in Figure \ref{fig:c2first}; only the first $4$ reachable sets are shown. We observed that $\ell_1$ could not be reached within a large enough number of iterations. Thus, we cannot reason about safety of this DFA transition under $\xi_1$. Thus, next we consider the controller $\xi_2$ and we constructed the sets $\bar{\mathfrak{R}}_t(\ccalX_0,\xi_2)$ (not shown) demonstrating that all obstacles are avoided while at $t=6$ the the reachable set is fully inside $\ell_2$. Thus, we verify that there exists $\boldsymbol\xi=\xi_2$, where $\xi_2$ is applied for $6$ steps, so that $\bbf_{\xi}\models\phi$, for all $\bbx_0\in\ccalX_0$.
%The generated ellipsoids for DFA transition now requires UAV to reach $\ell_1$ while avoiding $\ell_3$ and $\ell_4$ in Figure \ref{fig:c2first}. Notice that $\xi_1$ is a poorly trained NN controller, hence this transition cannot be verified as safe. But with the alternative to reach $\ell_2$, we can still verify this LTL formula to be success in Figure \ref{fig:c2second}

\begin{figure}[t]
\centering
\includegraphics[width=0.5\linewidth]{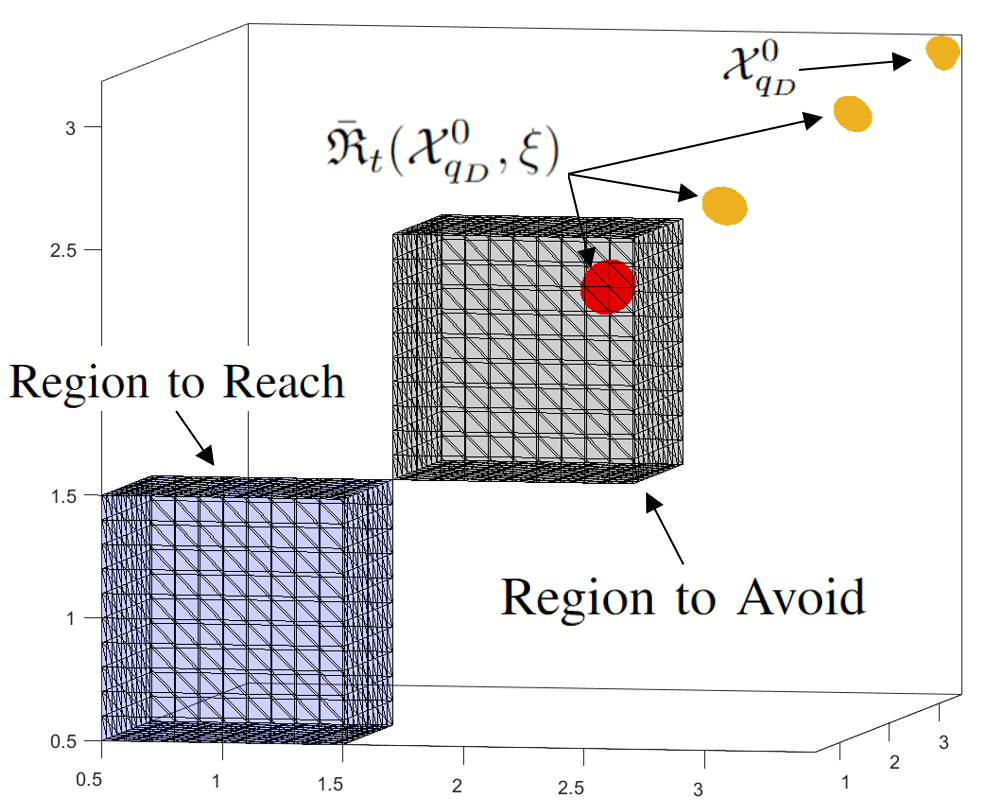}\vspace{-2mm}
\caption{Case Study III: Verification of the DFA transition from $q_D^0$ to $q_D^1$. The red ellipsoid corresponds to a reachable set being fully inside the region to avoid. }
\label{fig:case3}\vspace{-0.7cm}
\end{figure}
\normalsize

\textbf{Case Study III:} We revisit the LTL formula considered in Ex. \ref{ex2}-\ref{ex4} with DFA shown in Fig. \ref{fig:example}. The controllers $\xi_1$ and $\xi_2$ are the same as in case study I while $\ccalX_0$ is defined as an ellipsoid centered at $[3.4;3.4;3.1]$ with shape matrix of $\text{diag}[0.5^2;0.5^2;0.5^2]$. To check whether the DFA transition from $q_D^0$ to $q_D^1$ is safe, we check if the system can reach $\ell_1$ while avoiding $\ell_2$. By constructing the reachable sets, we notice that $\ell_1$ cannot be reached without entering $\ell_2$ first; see Fig. \ref{fig:case3}. Thus, we cannot find a $\boldsymbol\xi$ for $\phi$; see Rem. \ref{rem:complete}. %Note that Alg. \ref{algo1} is complete for this specific LTL formula; see Remark \ref{rem:complete}.

%\textcolor{red}{ I would suggest the following. Ideally, we need to show the DFA so that the reader can understand what `reach-avoid sub-tasks' we verify each time. But I am sure will not have space. So, I recommend we replace this LTL formula with the one we have considered in the example in the previous section for which we have already shown the DFA. This will save us some space. You will not need to re-train any NN but I believe you need to regenerate the figs. The reason is that going to $\ell_1$ requires to avoid $\ell_2$ which is not the case in the current LTL formula.}

\vspace{-0.3cm}
\section{Conclusion}
\vspace{-0.2cm}
This paper proposed a new method to design verified compositions of NN controllers for co-safe LTL tasks. We showed its efficiency in navigation tasks for aerial vehicles.

%\section*{APPENDIX}
%\section*{ACKNOWLEDGMENT}
\vspace{-0.25cm}
\bibliographystyle{IEEEtran}%{IEEEtran}
\bibliography{YK_bib.bib}

% Generated by IEEEtran.bst, version: 1.14 (2015/08/26)
\begin{thebibliography}{10}
\providecommand{\url}[1]{#1}
\csname url@samestyle\endcsname
\providecommand{\newblock}{\relax}
\providecommand{\bibinfo}[2]{#2}
\providecommand{\BIBentrySTDinterwordspacing}{\spaceskip=0pt\relax}
\providecommand{\BIBentryALTinterwordstretchfactor}{4}
\providecommand{\BIBentryALTinterwordspacing}{\spaceskip=\fontdimen2\font plus
\BIBentryALTinterwordstretchfactor\fontdimen3\font minus
  \fontdimen4\font\relax}
\providecommand{\BIBforeignlanguage}[2]{{%
\expandafter\ifx\csname l@#1\endcsname\relax
\typeout{** WARNING: IEEEtran.bst: No hyphenation pattern has been}%
\typeout{** loaded for the language `#1'. Using the pattern for}%
\typeout{** the default language instead.}%
\else
\language=\csname l@#1\endcsname
\fi
#2}}
\providecommand{\BIBdecl}{\relax}
\BIBdecl

\bibitem{gao2019reduced}
Q.~Gao, D.~Hajinezhad, Y.~Zhang, Y.~Kantaros, and M.~M. Zavlanos, ``Reduced
  variance deep reinforcement learning with temporal logic specifications,'' in
  \emph{10th ACM/IEEE International Conference on Cyber-Physical Systems},
  2019, pp. 237--248.

\bibitem{rubies2019classification}
V.~Rubies-Royo, D.~Fridovich-Keil, S.~Herbert, and C.~J. Tomlin, ``A
  classification-based approach for approximate reachability,'' in \emph{2019
  International Conference on Robotics and Automation (ICRA)}.\hskip 1em plus
  0.5em minus 0.4em\relax IEEE, 2019, pp. 7697--7704.

\bibitem{huang2017adversarial}
S.~Huang, N.~Papernot, I.~Goodfellow, Y.~Duan, and P.~Abbeel, ``Adversarial
  attacks on neural network policies,'' \emph{arXiv preprint arXiv:1702.02284},
  2017.

\bibitem{fazlyab2020safety}
M.~Fazlyab, M.~Morari, and G.~J. Pappas, ``Safety verification and robustness
  analysis of neural networks via quadratic constraints and semidefinite
  programming,'' \emph{IEEE Trans on Automatic Control}, 2020.

\bibitem{Dutta2017OutputRA}
S.~Dutta, S.~Jha, S.~Sankaranarayanan, and A.~Tiwari, ``Output range analysis
  for deep feedforward neural networks,'' in \emph{NASA Formal Methods
  Symposium}.\hskip 1em plus 0.5em minus 0.4em\relax Springer, 2018, pp.
  121--138.

\bibitem{huang2019reachnn}
C.~Huang, J.~Fan, W.~Li, X.~Chen, and Q.~Zhu, ``Reachnn: Reachability analysis
  of neural-network controlled systems,'' \emph{ACM Transactions on Embedded
  Computing Systems (TECS)}, vol.~18, no.~5s, pp. 1--22, 2019.

\bibitem{sun2019formal}
X.~Sun, H.~Khedr, and Y.~Shoukry, ``Formal verification of neural network
  controlled autonomous systems,'' in \emph{International Conference on Hybrid
  Systems: Computation and Control}, 2019, pp. 147--156.

\bibitem{Hu2020ReachSDPRA}
H.~Hu, M.~Fazlyab, M.~Morari, and G.~J. Pappas, ``Reach-sdp: Reachability
  analysis of closed-loop systems with neural network controllers via
  semidefinite programming,'' \emph{59th IEEE Conference on Decision and
  Control (CDC)}, pp. 5929--5934, December 2020.

\bibitem{Tran2020NNVTN}
H.-D. Tran, X.~Yang, D.~M. Lopez, P.~Musau, L.~V. Nguyen, W.~Xiang, S.~Bak, and
  T.~T. Johnson, ``{NNV}: The neural network verification tool for deep neural
  networks and learning-enabled cyber-physical systems,'' \emph{Computer Aided
  Verification}, vol. 12224, pp. 3 -- 17, 2020.

\bibitem{dutta2019reachability}
S.~Dutta, X.~Chen, and S.~Sankaranarayanan, ``Reachability analysis for neural
  feedback systems using regressive polynomial rule inference,'' in
  \emph{Proceedings of the 22nd ACM International Conference on Hybrid Systems:
  Computation and Control}, 2019, pp. 157--168.

\bibitem{ivanov2021verisig}
R.~Ivanov, T.~Carpenter, J.~Weimer, R.~Alur, G.~Pappas, and I.~Lee, ``Verisig
  2.0: Verification of neural network controllers using taylor model
  preconditioning,'' in \emph{International Conference on Computer Aided
  Verification}.\hskip 1em plus 0.5em minus 0.4em\relax Springer, 2021, pp.
  249--262.

\bibitem{sun2022formal}
S.~Sun, Y.~Zhang, X.~Luo, P.~Vlantis, M.~Pajic, and M.~M. Zavlanos, ``Formal
  verification of stochastic systems with relu neural network controllers,'' in
  \emph{2022 International Conference on Robotics and Automation (ICRA)}.\hskip
  1em plus 0.5em minus 0.4em\relax IEEE, 2022, pp. 6800--6806.

\bibitem{jothimurugan2021compositional}
K.~Jothimurugan, S.~Bansal, O.~Bastani, and R.~Alur, ``Compositional
  reinforcement learning from logical specifications,'' \emph{Advances in
  Neural Information Processing Systems}, vol.~34, 2021.

\bibitem{tasse2022skill}
G.~N. Tasse, D.~Jarvis, S.~James, and B.~Rosman, ``Skill machines: Temporal
  logic composition in reinforcement learning,'' \emph{arXiv preprint
  arXiv:2205.12532}, 2022.

\bibitem{neary2022verifiable}
C.~Neary, C.~Verginis, M.~Cubuktepe, and U.~Topcu, ``Verifiable and
  compositional reinforcement learning systems,'' in \emph{Proceedings of the
  International Conference on Automated Planning and Scheduling}, vol.~32,
  2022, pp. 615--623.

\bibitem{ivanov2021composelearning}
\BIBentryALTinterwordspacing
R.~Ivanov, K.~Jothimurugan, S.~Hsu, S.~Vaidya, R.~Alur, and O.~Bastani,
  ``Compositional learning and verification of neural network controllers,''
  \emph{ACM Trans. Embed. Comput. Syst.}, vol.~20, no.~5s, sep 2021. [Online].
  Available: \url{https://doi.org/10.1145/3477023}
\BIBentrySTDinterwordspacing

\bibitem{baier2008principles}
C.~Baier and J.-P. Katoen, \emph{Principles of model checking}.\hskip 1em plus
  0.5em minus 0.4em\relax MIT press Cambridge, 2008, vol. 26202649.

\bibitem{chen2022large}
S.~W. Chen, T.~Wang, N.~Atanasov, V.~Kumar, and M.~Morari, ``Large scale model
  predictive control with neural networks and primal active sets,''
  \emph{Automatica}, vol. 135, p. 109947, 2022.

\bibitem{leahy2016persistent}
K.~Leahy, D.~Zhou, C.-I. Vasile, K.~Oikonomopoulos, M.~Schwager, and C.~Belta,
  ``Persistent surveillance for unmanned aerial vehicles subject to charging
  and temporal logic constraints,'' \emph{Autonomous Robots}, vol.~40, no.~8,
  pp. 1363--1378, 2016.

\bibitem{kantaros2018text}
Y.~Kantaros and M.~M. Zavlanos, ``Stylus*: A temporal logic optimal control
  synthesis algorithm for large-scale multi-robot systems,''
  \emph{International Journal of Robotics Research}, 2020.

\bibitem{hahn}
E.~M. Hahn, M.~Perez, S.~Schewe, F.~Somenzi, A.~Trivedi, and D.~Wojtczak,
  ``Omega-regular objectives in model-free reinforcement learning,''
  \emph{TACAS}, 2018.

\bibitem{kantaros2022accelerated}
Y.~Kantaros, ``Accelerated reinforcement learning for temporal logic control
  objectives,'' in \emph{IEEE/RSJ International Conference on Intelligent
  Robots and Systems}, Kyoto, Japan, October 2022.

\bibitem{bozkurt2020control}
A.~K. Bozkurt, Y.~Wang, M.~M. Zavlanos, and M.~Pajic, ``Control synthesis from
  linear temporal logic specifications using model-free reinforcement
  learning,'' in \emph{2020 IEEE International Conference on Robotics and
  Automation (ICRA)}.\hskip 1em plus 0.5em minus 0.4em\relax IEEE, 2020, pp.
  10\,349--10\,355.

\bibitem{fuggitti-ltlf2dfa}
\BIBentryALTinterwordspacing
F.~Fuggitti, ``Ltlf2dfa,'' March 2019. [Online]. Available:
  \url{https://github.com/whitemech/LTLf2DFA}
\BIBentrySTDinterwordspacing

\bibitem{leahy2022fast}
K.~Leahy, A.~Jones, and C.~I. Vasile, ``Fast decomposition of temporal logic
  specifications for heterogeneous teams,'' \emph{IEEE Robotics and Automation
  Letters}, 2022.

\bibitem{andrei2017sqp}
N.~Andrei, ``A sqp algorithm for large-scale constrained optimization: Snopt,''
  in \emph{Continuous nonlinear optimization for engineering applications in
  GAMS technology}.\hskip 1em plus 0.5em minus 0.4em\relax Springer, 2017, pp.
  317--330.

\end{thebibliography}
%\vspace{-0.9cm}

\end{document}